\documentclass{article}
\usepackage[preprint]{neurips_2020}

\usepackage[utf8]{inputenc} % allow utf-8 input
\usepackage[T1]{fontenc}    % use 8-bit T1 fonts
\usepackage{hyperref}       % hyperlinks
\usepackage{url}            % simple URL typesetting
\usepackage{booktabs}       % professional-quality tables
\usepackage{amsfonts}       % blackboard math symbols
\usepackage{nicefrac}       % compact symbols for 1/2, etc.
\usepackage{microtype}      % microtypography
\usepackage{xcolor}         % colors

\usepackage{amsmath, amssymb, amsthm, bbm}
\usepackage{graphicx,tikz}
\usepackage{enumitem}
\usepackage{algorithm2e}
\usepackage{algpseudocode}
\usepackage{float}
\usepackage{bbm}
\usepackage{caption,hyperref, csquotes}
\usetikzlibrary{matrix}
\theoremstyle{theorem}
\theoremstyle{definition}
\newtheorem{definition}{Definition}
\newtheorem{theorem}{Theorem}
\newtheorem{proposition}{Proposition}
\newtheorem{corollary}{Corollary}

\newtheorem{lemma}{Lemma}
\theoremstyle{remark}

\newtheorem*{remark}{Remark}

\DeclareRobustCommand{\norm}[1]{\left\lVert #1 \right\rVert} %Norm
\DeclareRobustCommand{\norms}[1]{\lVert #1 \rVert} %Norm
\DeclareRobustCommand{\bb}[1]{\mathbb{#1}} 
\DeclareRobustCommand{\c}[1]{\mathcal{#1}}

\DeclareRobustCommand{\beef}{\vphantom{\sum}}

\DeclareRobustCommand{\set}[1]{\left\{#1 \right\}}
\DeclareRobustCommand{\Set}[2][]{\left\{#1 \, \middle|\, #2 \right\}}
\DeclareRobustCommand{\Sets}[2][]{\{#1 \, |\, #2 \}}
\DeclareRobustCommand{\bk}[2]{\left\langle #1,\,#2\right\rangle}

\DeclareRobustCommand{\bks}[2]{\langle #1,\,#2\rangle}

\DeclareRobustCommand{\inv}[1]{#1^{-1}}

\DeclareRobustCommand{\ep}{\varepsilon}

\DeclareRobustCommand{\inv}[1]{{#1}^{-1}}

\DeclareRobustCommand{\a}{\alpha}

\newcommand\restr[2]{{% we make the whole thing an ordinary symbol
  \left.\kern-\nulldelimiterspace % automatically resize the bar with \right
  #1 % the function
  \vphantom{\rule{0pt}{9pt}} % pretend it's a little taller at normal size
  \right|_{#2} % this is the delimiter
  }}

\DeclareMathOperator{\sign}{sign}

\usepackage{setspace}

\date{\today}

\begin{document}

\title{Classification as Direction Recovery: \\ Improved Guarantees via Scale Invariance}

\author{%
  Suhas Vijaykumar \\
  % Statistics and Data
  % Science Center \& Department of Economics\\
  MIT Economics \& Statistics \\
    \texttt{suhasv@mit.edu} 
  %Cambridge, MA 02139 \\
\And
  Claire Lazar Reich \\
%   Statistics and Data
% Science Center \& Department of Economics\\
  %Cambridge, MA 02139 \\
    MIT Economics \& Statistics\\
      \texttt{clazar@mit.edu} 
  % examples of more authors
}

\maketitle 

\begin{abstract}
Modern algorithms for binary classification rely on an intermediate regression problem for computational tractability. In this paper, we establish a geometric distinction between classification and regression that allows risk in these two settings to be more precisely related.
In particular, we note that classification risk depends only on the direction of the regressor, and we take advantage of this scale invariance to improve existing guarantees for how classification risk is bounded by the risk in the intermediate regression problem. Building on these guarantees, our analysis makes it possible to compare algorithms more accurately against each other and suggests viewing classification as unique from regression rather than a byproduct of it. While regression aims to converge toward the conditional expectation function  in location, we propose that classification should instead aim to recover its direction. 

%Modern algorithms for binary classification rely on an intermediate regression problem for computational tractability. The question of how to formulate this intermediate problem is hotly debated in both practice and theory. This paper takes another look at the statistical guarantees underpinning this debate. We note that classification risk depends only on the sign of a regressor and not its scale, and we take advantage of this scale invariance to improve existing bounds on classification risk so that algorithms can ultimately be compared more accurately against each other. Beyond these bounds, our analysis allows us to frame classification in a unique but natural way  

%Modern algorithms for binary prediction typically optimize a continuous proxy for the prediction error. However, the particular choice of proxy is the subject of ongoing debate. This paper takes a second look at the theoretical guarantees that underpin this debate. We show that conventional bounds based on the proxy fail to capture certain geometric aspects of the classification problem, particularly scale invariance. We give examples where the conventional bounds fail to capture the true convergence rate, and propose new, tighter bounds that exploit scale invariance. 
 \end{abstract}

\doublespacing

\section{Introduction}

The correct assignment of binary labels to data is a fundamental problem of machine learning. Practitioners naturally seek to minimize the portion of observations they misclassify, yet in practice, minimizing a loss function over binary labels and outcomes is computationally intractable. Therefore, modern approaches start with regression, which may be viewed as a convex relaxation of the original classification problem. They identify a real-valued function that minimizes a smooth \emph{surrogate risk criterion} pre-selected by the algorithm designer, and they then threshold resulting predictions to arrive at binary classifications. 

Several influential results in statistical machine learning relate the performance of a classifier to the performance of this intermediate regression procedure \citep{lugosi_bayes-risk_2004,zhang_statistical_2004,doi:10.1198/016214505000000907}. Attempts to attack this problem theoretically have, with few exceptions, proceeded by relating the classification risk to the surrogate risk, reducing the analysis to the better understood problem of stochastic optimization \citep{pmlr-v35-hazan14a,10.1162/089976604773135104}. Results on surrogate risk convergence, when combined with this analysis, produce theoretical guarantees that provide guidance on the choice of surrogate loss for classification tasks.

In this paper, we aim to improve this guidance by taking advantage of a fundamental geometric difference between classification and regression. In particular, we note that classification risk depends only on the direction of a regressor, whereas surrogate risk depends also on its scale. By taking advantage of the scale invariance of classification, we achieve tighter bounds relating classification risk to surrogate risk. The precision gained in these bounds can help practitioners better compare classification procedures and thus design better algorithms. 

Throughout our analysis, we reframe the problem of classification as “direction recovery.” To illustrate, suppose the conditional expectation function $\bb E [Y|X]$ may be represented by some $\beta^*$ a multidimensional vector space. Rather than aim to produce predictions close to $\beta^*$ in \emph{location}, as in traditional regression, we instead aim to produce predictions close to $\beta^*$ in \emph{direction}. We show that procedures designed to converge in direction to $\beta^*$ may achieve lower classification error than those that only seek to minimize regression error, and that upper bounds for existing procedures may be sharpened by studying convergence in direction.
We hope this perspective shows the potential of treating classification not only as a byproduct of regression, but as a unique problem deserving a tailored approach.

\subsection{Outline}

The paper proceeds in four main sections. In Section 2, we identify slack in existing bounds of classification risk and proceed to reduce the slack by introducing a notion of angle $\theta$ between a regressor and its optimal value. We characterize how a small angle $\theta$ minimizes excess classification risk. To achieve small angles in practice, we turn to surrogate loss minimization in Section 3. We show that regularized least squares obtains the optimal classifier when features are uncorrelated, and otherwise, study how regularization biases the predictor away from the optimal direction. Lastly, we present simulations in Section 4 that exemplify how surrogate loss minimization can fail or succeed to minimize the relevant angle.

\subsection{Related Work}
% Our work has its origin in the classical study of discrimination (using data to select among decision rules) as opposed to fully estimating the underlying distribution (which is closely related to prediction under the square or logistic loss, see e.g. \citet{vapnik_method_2006}). 
Recently, great attention has been paid to how bounds based on the excess risk alone can be systematically improved based on a margin condition, which says that the posterior probability of a positive label does not concentrate near one-half \citep{mammen_smooth_1999}. While these results are of great conceptual and practical importance, and similarly illustrate how bounds based on the surrogate risk can fail to adequately capture the classification problem, they rely on specific properties of the distribution at hand. In contrast, our results will focus primarily on the structure of classification problems in general, without attention to a specific class of distributions. Our results can, however, be strengthened by combining them with a margin condition, and we  illustrate this in the paper. 

Perhaps most similar to our work in spirit is the example of \citet{10.1007/11503415_20}, in which the excess classification risk decays exponentially fast, whereas the excess surrogate risk decays no faster than $O(1/n)$. While they make use of a margin condition and a specific class of distributions, particular attention is paid to the discrepancy between classification and stochastic optimization, and scale invariance of the optimal set of classifiers is used analytically. Like the authors, in our paper we give attention to the subtle aspects of classification that meaningfully affect classifier performance.
In the words of the authors, 

\begin{quote}
  ``In classification problems, there
are many relevant probabilistic, analytic and geometric parameters to play with
when one studies the convergence rates...  probably, we have
not understood to the end [the] rather subtle interplay between various parameters
that influence the behaviour of this type of classifier.''
\end{quote}

We believe that a powerful such parameter is the notion of direction we introduce in this paper.
%the notion of direction we introduce in our paper may be a helpful such parameter to  study and ultimately use to influence the behavior of classifiers.

Finally, canonical results relating the classification error to the surrogate risk can be found in \cite{doi:10.1198/016214505000000907}, and the results there are foundational to the analysis in this paper. 

\section{Bounding Excess Classification Risk}

In this section, we present a new bound on excess classification risk. To begin, we describe our setting and motivation for why these bounds are useful in practice. Then we show evidence of slack in existing bounds that can be substantially reduced, and finally, we reduce that slack to achieve tighter bounds. 

\subsection{Setting}

Consider a setting where features $X$ fix a linear conditional expectation function $f^*(X) = \bb E [Y|X] $ of a binary outcome $Y \in \{ -1,1\}$. A machine learner minimizes a surrogate convex loss function $\phi$ over data $(X,Y)$ to construct an estimate $f(X)$ of $f^*(X)$. The sign of $f$ fixes their classification decisions $\hat Y \equiv  \sign f$. Although they are minimizing $\phi$-loss, ultimately they care about classification loss: the probability that $f$ takes a different sign than $Y$.
 
Luckily, it has been shown that minimizing the convex surrogate loss successfully can constrain classification error. These guarantees are used in practice for machine learners to ultimately select their learning procedure. 
We discuss one contemporary approach to bound the classification error in the next section, as well as show slack in that approach which can be reduced to achieve tighter bounds. In practice, these tighter bounds can aid investigations of how learning procedures perform. 

\subsection{Finding slack in existing bounds for excess classification risk}

Contemporary bounds generally relate the excess classification risk to some increasing function $\psi$ of the excess $\phi$-risk, taking the form 
\begin{equation}
  \bb{P}(\sign f \ne Y) - \bb{P}(\sign f^* \ne Y) \le \psi(\bb{E}\phi(f,Y) - \bb{E}\phi(f^*,Y))  \label{eq:usual-bound}
\end{equation}
(cf. Thms. 1 and 3 of \citet{doi:10.1198/016214505000000907}). Since the slack is revealed from following the proof behind this bound, we  sketch the main idea.  

The argument on which this bound is based relies first on fixing a value of $f^*$ and then computing the associated excess classification risk of an arbitrary $f$, 
\begin{align}
  \bb{P}(\sign f \ne Y) &- \bb{P}(\sign f^* \ne Y) 
   = \bb{E}\left[|f^*| \mathbbm{1}\{\sign f \ne \sign f^*\} \right] 
 \label{eq:cls-risk-step}
\end{align}
 
 \citet{doi:10.1198/016214505000000907} show how to bound the step function (\ref{eq:cls-risk-step}) by a smooth convex function based on $\phi$-risk. We illustrate their bound in Figure \ref{fig:slack} for a given $f^*$ and shade the associated slack in green (when $\sign f = \sign f^*$) and in yellow (when $\sign f \ne \sign f^*$). The way that we will ultimately reduce this slack is by noting that the LHS of  (\ref{eq:cls-risk-step}) depends only on whether $f^*$ and $f$ share the same sign. 
 Therefore, we have the opportunity to rewrite the bound (\ref{eq:usual-bound}) in terms of a predictor $g$ that i) satisfies $\sign g = \sign f$ so that the LHS of  (\ref{eq:usual-bound}) is unchanged but ii) that corresponds to a tighter convex bound on the RHS of  (\ref{eq:usual-bound}). To do so, we will choose $g$ among the rescalings of $f$, i.e., among all vectors pointing in the direction of $f$. 
 %Therefore, whereas the RHS of the \citet{doi:10.1198/016214505000000907} bound (\ref{eq:usual-bound}) was based on a fixed arbitrary $f$, we will optimize over predictors $g$ satisfying $\sign g = \sign f$. 

\begin{figure}[b]
  \centering
\includegraphics[width=.75 \columnwidth]{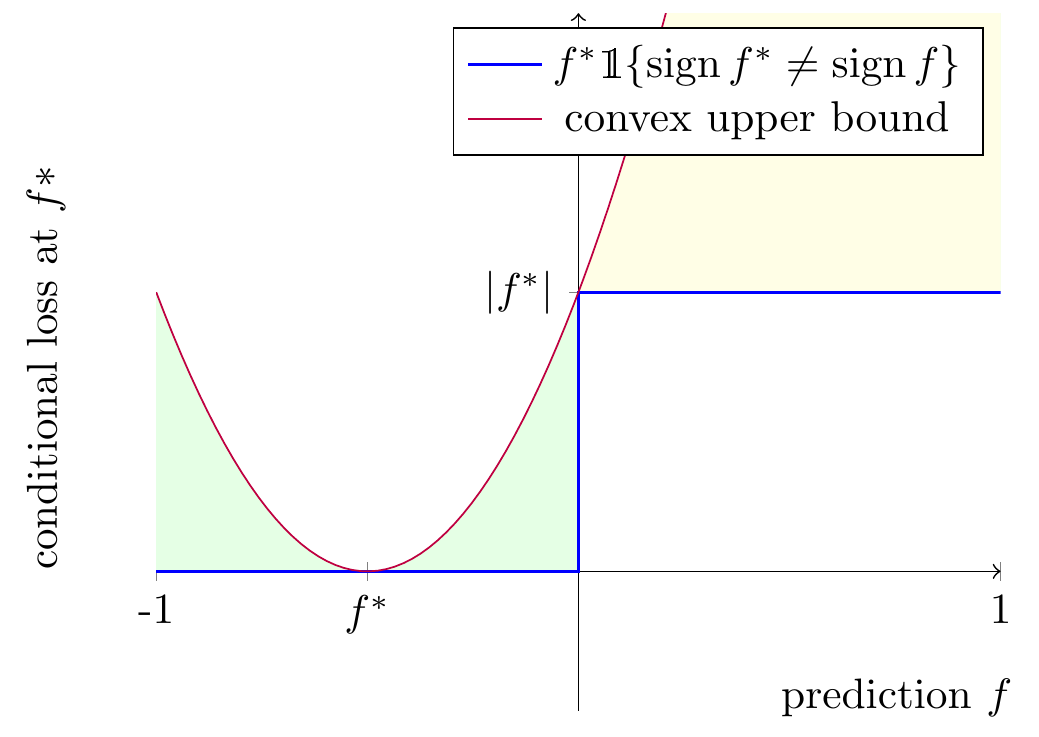}
\caption{\label{fig:slack}Illustration of upper bound for a given $f^*$. The green and yellow regions correspond to slack that we aim to reduce.}
\end{figure}

\subsection{Tightening the bound: a first example}
In order to illustrate the relevance of the direction of $f$ to its associated classification loss, we present a visual example. In particular,
%In order to explore in a simple example how the excess surrogate loss can fail to capture prediction error due to slack in the bound (\ref{eq:usual-bound}), 
we consider the problem of learning a linear classification rule in the presence of features that satisfy a notion of symmetry called rotational invariance.  
\begin{definition}
The law of $X$ is \emph{rotation invariant} if it satisfies $\bb{P}(X \in S) = \bb{P}(RX \in S)$ for any measurable set $S$ and any rotation $R$ of its coordinates.
\end{definition}

Note that rotational invariance is a stronger condition than uncorrelated features. 
This property allows us to exactly characterize the probability that a linear predictor $\tilde \beta$ yields a different classification from the optimal $\beta^*$, based on the angle $\theta(\tilde \beta, \beta^*)$ between them. The probability grows in $\theta$ as follows.
%When it holds, then we can characterize the prediction error exactly in terms of the \emph{angle} between a linear predictor $\beta$ and the optimal classifier $\beta^*$, which we denote $\theta(\beta^*, \beta)$.
\begin{figure}
  \centering
\includegraphics[width=.5 \columnwidth]{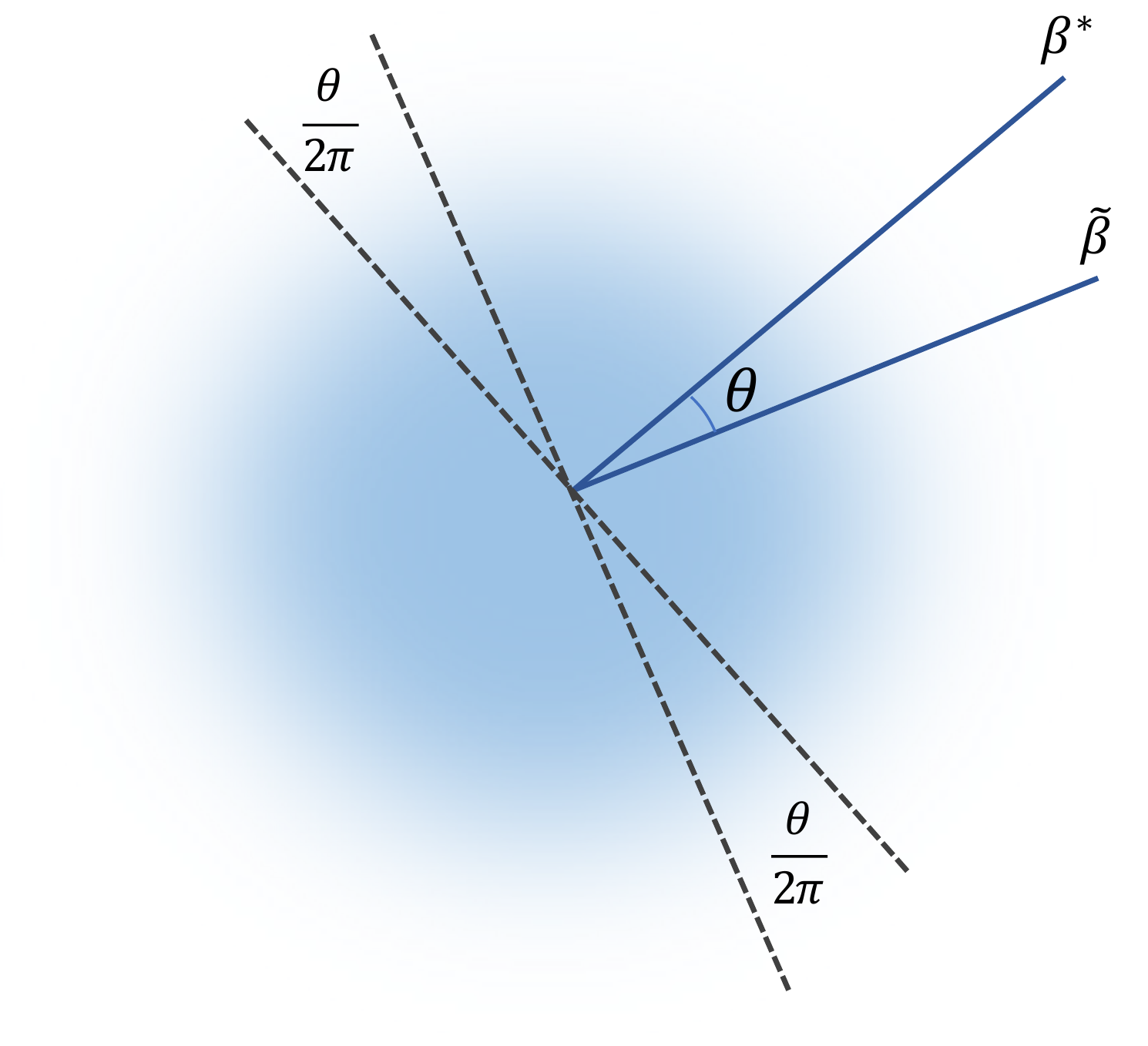}
\caption{\label{fig:theory_diag_3} Top-down view of projection $Z$ on plane spanned by $(\tilde \beta, \beta^*)$ Dashed lines correspond to the classification boundaries associated with each $\beta \in \{ \tilde \beta, \beta^*\}$. Points to the right are positively classified by $\beta$ while those to the left are negatively classified. Thus, the two sectors between the dashed lines (with combined measure of $\frac{\theta}{\pi}$) designate observations that are classified differently by $\tilde \beta$ and $\beta^*$.} 
\end{figure}

\begin{lemma}
  If the law of $X$ is rotation invariant, we have
  \begin{equation}
  \bb{P}(\mathrm{sign}\bks{\beta^*}{X}\ne \mathrm{sign}\bks{\tilde\beta}{X}) = \frac{\theta(\tilde \beta, \beta^*)}{\pi} \label{rot-inv-eq}
  \end{equation}
\end{lemma}
\begin{proof}
  We have 
  \begin{align*}
     \bb{P}(\mathrm{sign}\bks{\beta^*}{X}\ne \mathrm{sign}\bks{\tilde\beta}{X}) 
    = \bb{P}\left(\mathrm{sign}\bk{\tilde\beta}{\frac{X}{\norms{X}_2}}\ne \sign \bk{\beta^*}{\frac{X}{\norms{X}_2}}\right)
  \end{align*}
Now let $Z$ denote the projection of $X/\norms{X}_2$ onto the plane spanned by $(\tilde \beta, \beta^*)$. We know $Z$ is uniformly distributed on the circle by rotation invariance, and the sign associated with $\tilde \beta$ and with $\beta^*$  will differ precisely when $Z$ belongs to a subset of the circle of measure $\theta/\pi$. This is illustrated in Figure \ref{fig:theory_diag_3}, which visualizes the projection $Z$ from above. %\textcolor{red}{[Add diagram]}
\end{proof}

The key insight which leads to the angle $\theta(\beta^*,\beta)$ on the RHS of (\ref{rot-inv-eq})
 is that the classification rules $x \mapsto \sign\bk{\beta}{x}$ and $x \mapsto \sign\bk{\beta^*}{x}$ are invariant to rescaling the linear predictors $\beta$ and $\beta^*$. The angle, which corresponds to the distance between $\beta/\norm{\beta}$ and $\beta^*/\norm{\beta^*}$ along the surface of the unit sphere, emerges as a natural, \emph{scale invariant} measure of the distance between the two predictors. In the following section, we will see that this intuition extends far beyond the simple case of rotationally invariant features.

\subsection{The general angle criterion}

In the general case, the situation is more nuanced. However, the key insight that the classifier $\mathbbm{1}\{f \ge 0\}$ is invariant to rescaling $f$ remains.

 To begin, we need to know how to actually express the angle $\theta$ between two vectors $u$ and $v$. In the following lemma, we use a geometric argument to write $\sin \theta$ as the minimum distance between a rescaled $u$ and a normalized $v$.

 \begin{lemma} \label{lemma-sin-theta} The angle $\theta(u,v)$ between vectors $u, v \in \bb{R}^d$ with $\bk{u}{v} > 0$ satisfies
  \[\sin \theta(u,v) =  \inf_{t \ge 0} \norm{tu - \frac{v}{\norm{v}_2}},\]
  and $\theta(-u,v) = \pi - \theta(u,v)$ for all $u,v\in \bb{R}^d$.
\end{lemma}

\begin{proof} As seen in Figure \ref{fig:theory_diag_2}, the distance $\norm{tu - \frac{v}{\norm{v}_2}}$ is minimized when $t=t^*$, where $t^*u$ is equal to the orthogonal projection of $v/\norms{v}$ onto the line spanned by $u$. As seen, $(v/\norms{v}, 0,t^*u)$ forms a right triangle with hypotenuse of length $1$, and $(t^*u, v/\norms{v})$ is opposite the angle $\theta(u,v)$.
% \textcolor{blue}{Meanwhile the Taylor expansion of $\sin \theta = \theta - \frac{\theta^3}{3!} +  \frac{\theta^5}{5!} - \cdots$ implies that $\sin \theta \downarrow 0$ iff $\theta \downarrow 0 $ and the convergence rates are asymptotically equivalent}.
\end{proof}
\begin{figure}
  \centering
\includegraphics[width=.6 \columnwidth]{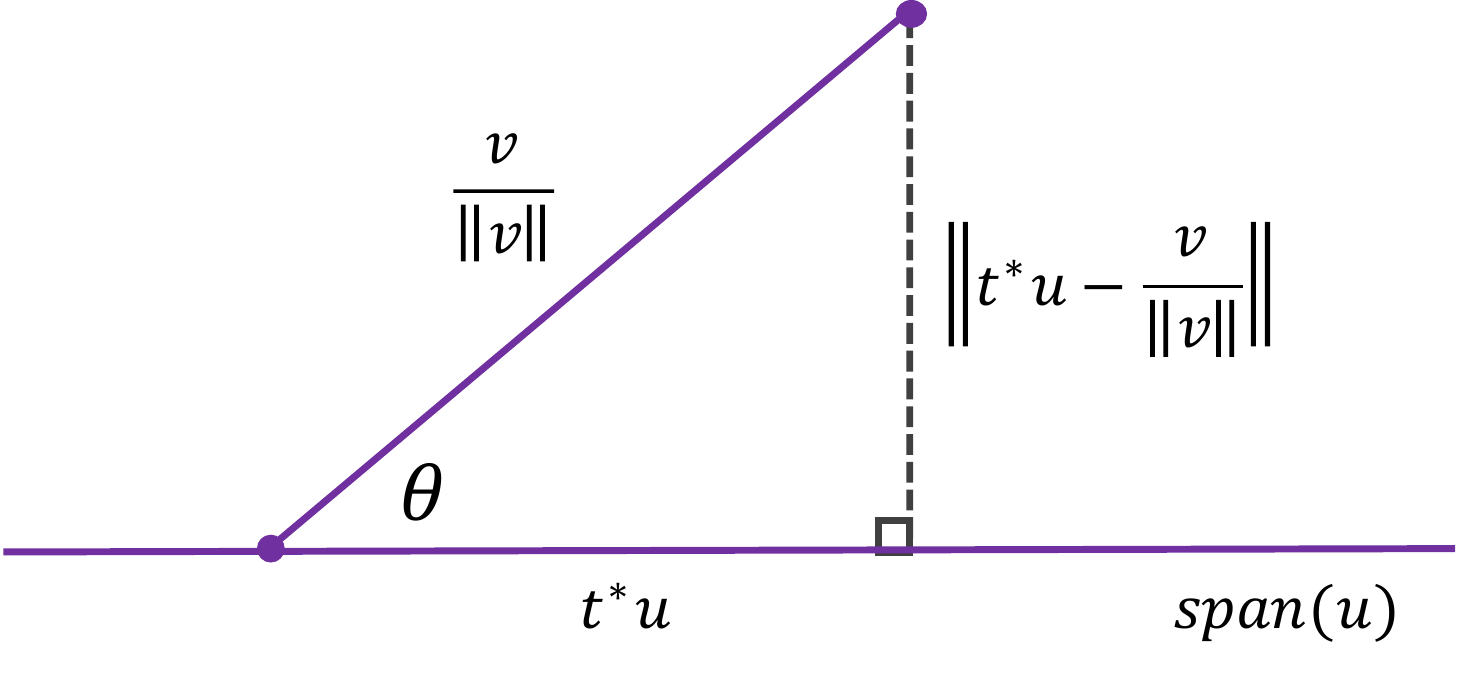}
\caption{\label{fig:theory_diag_2}  $\sin \theta$ is the shortest distance between $span(u)$ and $\frac{v}{\norm v}$.} 
\end{figure}

Now that we no longer require the law of $X$ to be rotation invariant, we must deal directly with  $L^2(\bb P)$. This requires us to define the relevant angle of a predictor with respect to the norm $\norm{f}_{2,\bb P} = \bb{E}[f^2]^{\frac 1 2}$ in the probability space. Motivated by Lemma \ref{lemma-sin-theta}, we generalize our notion of the relevant angle $\theta_{2,\bb{P}}(\beta,\beta^*)$ to $L^2(\bb P)$ by the relation
% In fact, rotation invariance is not actually required for a result of this form to go through, provided we are willing to replace $\theta(\beta,\beta^*)$ by $\theta_{2,\bb{P}}(\beta,\beta^*)$. Here $\theta_{2,\bb P}$ is the angle measured with respect to the Euclidean structure in $L^2(\bb P)$ given by the norm $\norm{f}_{2,\bb P} = \bb{E}[f^2]^{\frac 1 2}$. \textcolor{red}{[Why this norm? Because we care about square loss? Is the idea here that you need to weigh by probability distribution since you lost rotational invariance?]} This angle may be defined for functions $u$ and $v$ with $\bb{E}[uv]>0$ by the relationx
\begin{equation}
  \sin\theta_{2,\bb P}(u,v) = \inf_{t > 0}\norm{tu - \frac{v}{\norm{v}_{2,\bb P}}}_{2,\bb P}, \label{eq:general-sin}
\end{equation}
when $ \bb{E}[uv]>0$  and  $\theta_{2,\bb P}(u,v) = \pi - \theta_{2,\bb P}(-u,v)$  when  $ \bb{E}[uv]<0$.

%In light of Lemma \ref{lemma-sin-theta}, \eqref{eq:general-sin} can be seen as a natural generalization of the angle to $L^2(\bb P)$.

{Equipped with the suitable notion of angle, we can establish a main result. Here we bound the excess classification risk of an arbitrary predictor $f$ according to the direction of $f$. Note that while the $L^2(\bb P)$ distance and the square loss are essential ingredients in its proof, the result applies to \emph{any} predictor, however it is obtained. } \
\begin{theorem}\label{thm:sin-bound}
  Let $f^* = \bb{E}[Y|X]$. Then, the excess classification risk is bounded as 
  \[\bb{P}(Y \ne \sign f) - \bb{P}(Y \ne \sign f^*) \le \norm{f^*}_{2,\bb{P}} \sin \theta_{2,\bb P}(f,f^*).\]
\end{theorem}

We prove this result using the canonical bound given by Theorem 1 in \citet{doi:10.1198/016214505000000907}, which relates the classification risk in excess of $f^*$ to the excess surrogate loss. Stated for the square loss, the result reduces to the following. 
\begin{theorem}[{\citet[Thm. 1]{doi:10.1198/016214505000000907}}]\label{thm:bartlett-square}
  \begin{equation*}
    \bb{P}(Y \ne \sign f) - \bb{P}(Y \ne \sign f^*) \le \norm{f-f^*}_{2,\bb P}\label{eq:bartlett-square}
  \end{equation*}
\end{theorem}

We use the canonical theorem to prove our new result.

\begin{proof}[Proof of Theorem \ref{thm:sin-bound}]
  Let $C_f$ denote the convex cone of functions $g$ satisfying
  $\sign g = \sign f$ almost surely. 
  For any $g \in C_f$ we can apply Theorem \ref{thm:bartlett-square} to obtain 
  \begin{align*}
    \bb{P}(Y \ne \sign f) - \bb{P}(Y \ne \sign f^*)
    &= \bb{P}(Y \ne \sign g) - \bb{P}(Y \ne \sign f^*) \\
    &\le \norm{g-f^*}_{2,\bb P}.
  \end{align*}
  Optimizing over the bounds obtained in this manner yields
  \begin{equation}
    \bb{P}(Y \ne \sign f) - \bb{P}(Y \ne \sign f^*) \le \inf_{g \in C_f}\left\{\beef\norm{g-f^*}_{2,\bb P}\right\}
  \end{equation}
  While we cannot tractably minimize over all $g$ in the cone $C_f$, we can minimize over the $g$ that rescale $f$, noting that these rescalings satisfy $\Set[tf]{t > 0} \subset C_f$. This gives us the bound  %\textcolor{red}{[Add intuition]}
  \begin{align*}
    \bb{P}(Y \ne \sign f) - \bb{P}(Y \ne \sign f^*)
    &\le \inf_{t > 0 }\left\{\beef\norm{tf-f^*}_{2,\bb P}\right\} \\
    &= \norm{f^*}_{2, \bb P} \inf_{t > 0}\left\{\norm{\frac{tf}{\norm{f^*}_{2, \bb P}}-\frac{f^*}{\norm{f^*}_{2, \bb P}}}_{2,\bb P}\right\}.
    \intertext{Making the change of variable $t'= \norm{f^*}_{2, \bb P}t$ gives}
    &= \norm{f^*}_{2, \bb P} \inf_{t' > 0}\left\{\norm{t'f-\frac{f^*}{\norm{f^*}_{2, \bb P}}}_{2,\bb P}\right\} \\
    &= \norm{f^*}_{2, \bb P} \sin \theta_{2,\bb P}(f,f^*),
  \end{align*}
  by our definition of $\theta_{2, \bb P}$. This is what we aimed to show. 
\end{proof}
In fact, \citet{doi:10.1198/016214505000000907} proved stronger versions of Theorem \ref{thm:bartlett-square} under the \emph{low-noise condition} (also sometimes called a \emph{margin condition}), and the same machinery can be applied to yield stronger versions of our Theorem \ref{thm:sin-bound}. To state these, we first introduce the low-noise condition. It characterizes the extent to which the best prediction of the outcome is close to the classification boundary.
\begin{definition}
Given $\a \in [0,1]$ the pair $(X,Y)$ is said to satisfy the $\a$-noise condition if for some $C > 0$, $f^*=\bb{E}[Y|X]$ satisfies
\begin{equation*}
  \bb{P}\left(\left|f^* \right| < \ep \right) \le C\ep^{\a/(1-\a)},  
\end{equation*}
for all sufficiently small $0 < \ep < c$. 
\end{definition}
Under the above condition, \citet{doi:10.1198/016214505000000907} proved the following improvement on their bound.
\begin{theorem}[{Special case of \citet[Theorem 3]{doi:10.1198/016214505000000907}}]
  Suppose $(X,Y)$ satisfies the $\a$-noise condition with constants $c, C > 0$. Then 
  \[\bb{P}(\sign f \ne Y) - \bb{P}(\sign f^* \ne Y) \le \frac{\norm{f-f^*}_{2,\bb P}^{1 + \a}}{4c'}\] for some $c'$ which depends only on $c$ and $C$.\label{thm:bartlett-margin}
\end{theorem}
Repeating the proof of Theorem \ref{thm:sin-bound}, and replacing our use of Theorem \ref{thm:bartlett-square} with the improved Theorem \ref{thm:bartlett-margin}, gives the following improved result.
\begin{theorem}\label{thm:sin-margin}
  Suppose $(X,Y)$ satisfies the $\a$-noise condition with constant $c > 0$. Then, for the same $c'$ appearing in Theorem \ref{thm:bartlett-margin}, it holds that 
  \begin{align}
    \bb{P}(\sign f \ne Y) - \bb{P}(\sign f^* \ne Y)
    &\le \inf_{g \in C_f} \left\{ \frac{\norm{g-f^*}_{2,\bb P}^{1 + \a}}{4c'} \right\} \\
    &\le \frac{\norm{f^*}_{2, \bb P}^{1+\a}}{4c'} \left( \sin \theta_{2,\bb P}(f,f^*)\right)^{1+\a}
  \end{align}
\end{theorem}

\begin{remark}
  An interesting aspect of the bounds in Theorem \ref{thm:sin-margin} and Theorem \ref{thm:sin-bound} is that they are \emph{never weaker} than the corresponding bounds of \citet{doi:10.1198/016214505000000907}, which relate classification risk to the excess square loss, on which they are based. As such, since $\norm{f^*}_{2,\bb P}$ is a problem-invariant constant, procedures that are tailored to minimization of $\theta_{2,\bb{P}}(f,f^*)$ will yield stronger bounds than those based only on control of the excess mean squared error. 
\end{remark}

\begin{remark}
For the rotationally invariant case, we show in the appendix that excess classification error is given precisely by 
  \[\frac{1}{\pi}\left(\int_{0}^\theta \sin t\, dt \right)\bb{E}|\bk{X}{\beta^*}|.\] 
We explain how this implies a convergence rate of $\frac{1}{n}$, whereas the standard bound by \citet{doi:10.1198/016214505000000907} only guarantees a rate of $\frac{1}{\sqrt n}$. Therefore, we see that rotation invariant linear classification produces fast rates, even without imposition of a margin condition. This demonstrates how an angle-based analysis of learning procedures can improve convergence guarantees from traditional bounds.  
\end{remark}

%\clearpage 

%\clearpage 
\section{Relationship between surrogate loss minimization and $\theta$}

\subsection{Defining classification calibration}
Thus far, we have related excess classification risk to the angle $\theta $ between a predictor $\tilde \beta$ and the optimal $\beta^*$, showing that the excess risk is guaranteed to be small when $\theta$ is small. Now we turn our attention to how $\theta$ is actually determined. In particular, we consider procedures that minimize a surrogate loss function $\phi$ over a set $S\subset \bb{R}^d$ of linear predictors and present guarantees on their maximum associated values of $\theta$. 

Procedures that are guaranteed to achieve $\sin(\theta) = 0$ in the population are of particular interest, as they converge to the optimal classifier. We call these ``classification calibrated.'' In the case of minimizing surrogate risk over $S$, we define this special trait as follows.  

\begin{definition} 
A procedure that minimizes the $\phi$-risk $\bb E [\phi (Y, \langle \beta, X \rangle )]$ over $S$ is \emph{classification calibrated} if its constrained minimizer $\tilde \beta$ is also a global minimizer of the classification risk $\bb P [ \sign \langle \beta, X \rangle \neq Y]$.
\end{definition} 

Note that this definition does not require the procedure to identify the global minimizer of the $\phi$ risk. In fact, in our simulations we will provide examples of when the global minimizer of $\phi$-risk is not contained in $S$, but the constrained minimizer in $S$ yields the global minimum of classification risk nonetheless. 

In this section, we study classification calibration in settings with well-specified models that are regularized so that $S$ is a ball of positive radius $r$, i.e.,
\[S = \Set[v \in \bb R^d]{\norm v \le r}.\]

In the following lemma, we start by noting that this choice of $S$ contains a global minimizer of the classification risk. The question therefore becomes, when does minimizing the surrogate loss $\phi$ within $S$ identify this global minimizer?
\begin{lemma}
$S$ contains a global minimizer of the classification risk.
\end{lemma}
\begin{proof}
Since $\mathrm{sign}\bk{\beta}{X} = \mathrm{sign}\bk{c \beta}{X}$ for any $c >0$, recall that the classifications associated with any given $\beta$ are invariant to rescaling $\beta$. This is illustrated in Figure \ref{fig:theory_diag_1}. Since $S$ contains a neighborhood of the origin, it follows that it is guaranteed to contain a rescaled $\tilde \beta = c \beta^* \in S$ of some global minimizer $\beta^*$ of the classification risk. 
\end{proof}

\begin{figure}
  \centering
\includegraphics[width=.6 \columnwidth]{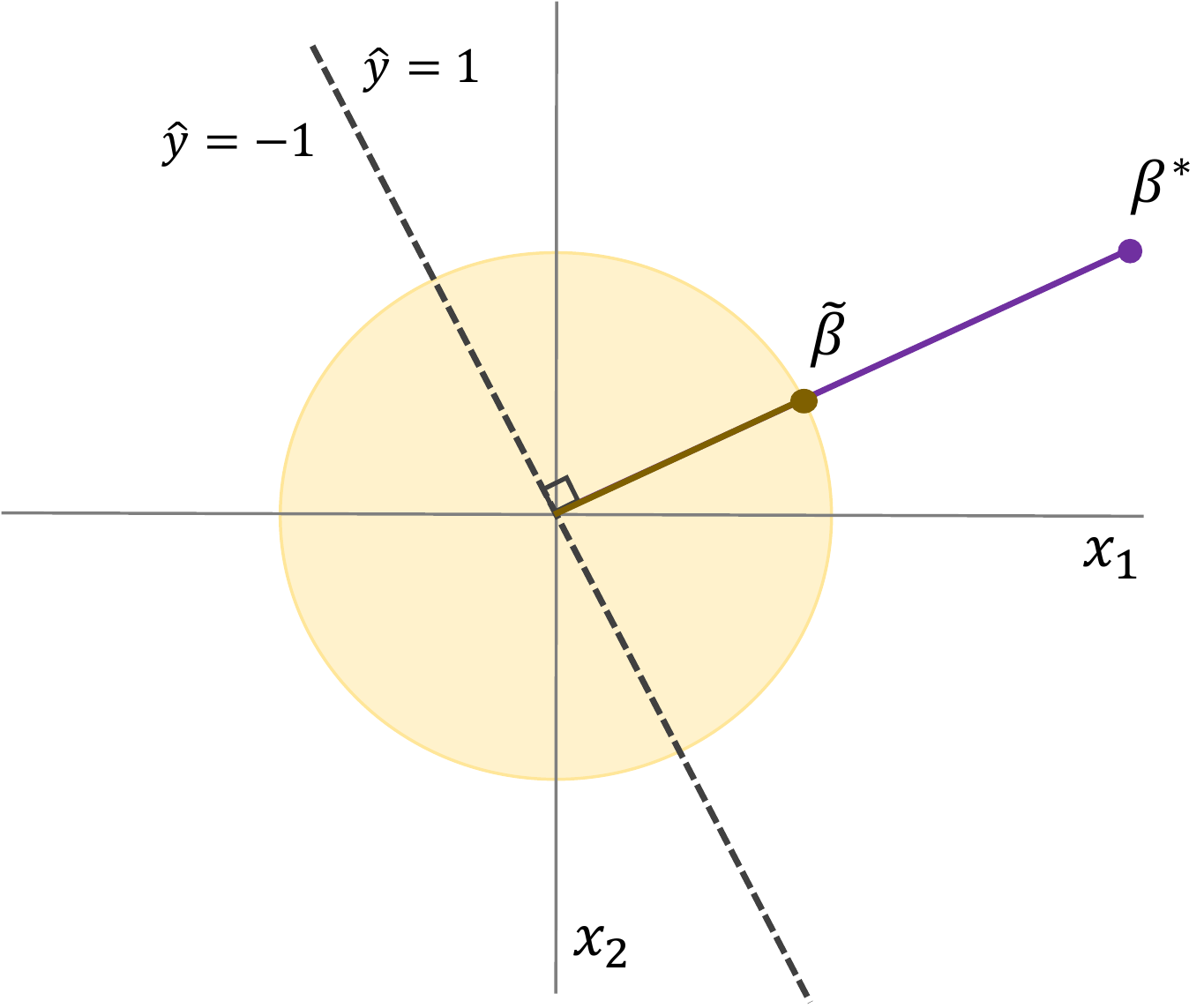}
\caption{\label{fig:theory_diag_1} Example of a classification calibrated predictor. The optimal regressor $\beta^*$ corresponds to the  same classifications as the minimizer $\tilde \beta$ in the orange ball depicting $S$. Instances to the right of the dashed line are classified as positive and those to the left of the dashed line are classified as negative. } 
\end{figure}

\subsection{Studying $\theta$ under Square Loss Minimization}
In this section, we will investigate the case where $\phi$ is the square loss $\phi(Y, f(X)) = (Y-f(X))^2$. We consider de-meaned features $X$ that may or may not be correlated and then  characterize the population square loss in terms of their covariance matrix $\Sigma = \bb{E}[XX^\top]$. Then, we will bound the angle between the constrained population minimizer $\tilde \beta$ and the unconstrained population minimizer $\beta^*$ in terms of $\Sigma$, showing that when features are uncorrelated, the angle is 0. Finally, we will investigate whether the excess misclassification risk can be controlled in terms of $\beta^*$, $\tilde \beta$, and $\Sigma$ alone. 

\subsubsection{Characterizing the population loss} We will show that the population loss associated with a linear predictor $\beta$ can be expressed in terms of $\Sigma$. We begin with the following lemma as an intermediate step to computing the mean square loss.

\begin{lemma}\label{lem:quad-form-sigma}
  If $\bb{E}[XX^\top]=\Sigma$ then $\bb{E}\bk{X}{v}^2=v^\top\Sigma v$.
\end{lemma}
\textit{Proof in appendix.}

This result makes it possible for us to express the mean square loss associated with a linear predictor $\beta$ in terms of $\Sigma$ and $\beta^*$, the orthogonal projection of $Y$ onto the span of the features. We will see then that choosing $\beta$ to minimize the mean square loss corresponds to minimizing the expression $ (\beta^* - \beta)^\top \Sigma (\beta^* - \beta)$.
\begin{lemma}\label{lem:pop-loss}
Let $\beta^*$ be the orthogonal projection of $Y$ onto the span of the features, $\Set[\bks{\gamma}{X}]{\gamma \in \bb{R}^d}$ in $L^2(\bb P)$.  %Best linear predictor?
  Then 
  \begin{align*}
    \bb{E}(Y-\bks{\beta}{X})^2 
    &= \bb{E}(Y-\bks{\beta^*}{X})^2 + \bb{E}(\bks{\beta^*}{X}-\bks{\beta}{X})^2 \\
    &= C + (\beta^* - \beta)^\top \Sigma (\beta^* - \beta),
  \end{align*}
   where $C$ is a constant independent of $\beta$.
\end{lemma}

\textit{Proof in appendix.}

These lemmas have particularly useful implications in cases where the features are uncorrelated and standardized, that is, when $X$ is isotropic according to the following definition. %what about when that's not the case? 
\begin{definition}
  A random vector $X \in \bb{R}^d$ is called \emph{isotropic} if $\bb{E}[XX^\top] = I$.
\end{definition}
When this holds, then the following corollary proves that minimizing the mean square loss in the population corresponds to choosing $\beta$ to minimize its distance from $\beta^*$.
\begin{corollary}
 The $\beta$ that minimizes $ \bb{E}(Y-\bks{\beta}{X})^2$ also minimizes $\norm{\beta^* - \beta}_2^2$. %\textcolor{red}{[Why do we care about this?]} 
 \label{cor-beta-min-distance}
  \end{corollary}
  
  \textit{Proof in appendix}
  
\subsubsection{Bounding the angle}

We now shift our attention to bounding the angle between a linear predictor $\beta$ and the best linear predictor $\beta^*$.
% To start, we will introduce a useful } proxy for the angle between vectors $u$ and $v$,
% \begin{equation}
%   \Delta(u,v) \textcolor{blue}{\equiv} \inf_{t \in \bb{R}}\norm{tu - \frac{v}{\norm{v}_2}}.
% \end{equation}
% This is motivated by the following.
% \begin{lemma} The angle $\theta(u,v)$ between vectors $u$ and $v$ satisfies
%   \[\sin \theta(u,v) =  \inf_{t \in \bb{R}}\norm{tu - \frac{v}{\norm{v}_2}}.\]
%   It follows from the Taylor expansion of $\sin \theta$ that, $\theta(u,v) \downarrow 0$ iff $\Delta(u,v) \downarrow 0$ and the convergence rates are asymptotically equivalent.  
%   \textcolor{red}{[Why mention Taylor expansion here and not just in the proof?]}
% \end{lemma}
% \begin{proof}The distance \textcolor{blue}{$\norm{tu - \frac{v}{\norm{v}_2}}$} is minimized when $t=t^*$, where $t^*u$ is equal to the orthogonal projection of $v/\norms{v}$ onto \textcolor{blue}{the line spanned by $u$}. Thus, $(v/\norms{v}, 0,t^*u)$ forms a right triangle with hypotenuse of length $1$, and $(t^*u, v/\norms{v})$ is opposite the angle $\theta(u,v)$. \textcolor{red}{[Add diagram] \textcolor{blue}{Meanwhile the Taylor expansion of $\sin \theta = \theta - \frac{\theta^3}{3!} +  \frac{\theta^5}{5!} -...$ implies that $\sin \theta \downarrow 0$ iff $\theta \downarrow 0 $ and the convergence rates are asymptotically equivalent}.}
% \end{proof}

We see that in the isotropic case, minimizing square loss recovers a $\tilde \beta$ whose angle with $\beta^*$ is 0. That is, minimizing square loss gives the optimal classifier. 

\begin{proposition} \label{prop:sq-loss-isotropic}
 %\textcolor{red}{[Why not say, the angle between $\tilde \beta$ and $\beta^*$ equals 0 instead of inf statement?]}
  If $X$ is isotropic, then the minimizer $\tilde \beta \in S$ of the square loss $\bb{E}(Y-\bks{\beta}{X})^2$ satisfies 
  \[ \sin \theta (\tilde \beta, \beta^*) = 0 \]
%  \[\inf_{t \in \bb{R}}\norm{t\tilde \beta - \frac{\beta^*}{\norms{\beta^*}}_2} = 0\]
\end{proposition}
\noindent This follows easily from the following, more general result given a feature covariance matrix $\Sigma$.
\begin{theorem} \label{thm:sq-loss-general}
  In general, the minimizer  $\tilde \beta \in S$ of the square loss satisfies
  \[\sin \theta (\tilde \beta, \beta^*) 
  %\inf_{t \in \bb{R}}\norm{t\tilde \beta - \frac{\beta^*}{\norms{\beta^*}}_2} 
  \le \inf_{a \ge 0} \norm{a\Sigma - I}_{\mathrm{op}}\]
\end{theorem}

\textit{Proof in appendix.}

\subsection{Studying $\theta$ under General Loss Minimization}

When we consider minimizing a general surrogate loss function $\phi$ in a setting where the law of $X$ is rotation invariant, then we can guarantee convergence to the optimal classifier. 

\begin{proposition} Suppose that the following conditions hold. \label{prop:general-loss}
  \begin{enumerate}[label=(\roman*)]
    \item The law of $X$ is rotation invariant. \label{it:rotinv}
    \item $\bb{P}(Y=1|X) = \eta(\bks{X}{\beta^*})$ for some $\beta^* \in \bb{R}^d$. \label{it:suff}
    \item The loss function $\phi(Y,f)$ is convex in $f$. \label{it:conv}
  \end{enumerate}
  Then the constrained minimizer $\tilde \beta$ of $\beta \mapsto \bb{E}\phi(Y,\bks{\beta}{X})$ subject to $\norm{\beta} \le r$ is unique and satisfies $\tilde \beta = c\beta^*$ from some $c \in \bb{R}$.
\end{proposition}
\textit{Proof in appendix.}
In the simulations in the following section, we present evidence that this result cannot be completely relaxed without further assumptions.

%\section{Directional Recovery}

%\textcolor{red}{[Insert]}

%%%%% SIMULATION SECTION

\section{Application}
\begin{figure}
  \centering
\includegraphics[width=\columnwidth]{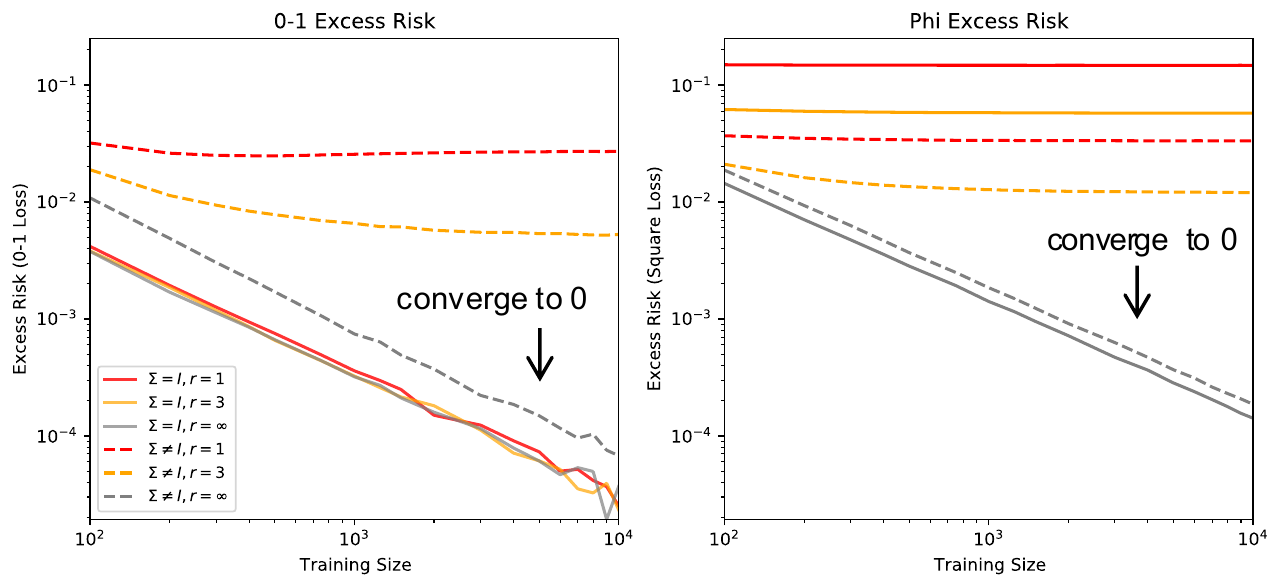}
\caption{\label{fig:sqloss_plot} Minimizing square loss when $p^*$ is linear in $X$. Gray lines correspond to correctly specified models without binding norm constraints, and they are therefore associated with excess risks that converge to zero in panels (a) and (b). In panel (a) we see that the misspecified models trained on uncorrelated features (solid red and solid orange lines) have excess 0-1 risk converging to zero despite the fact that panel (b) shows they do not recover a globally minimized $\phi$-risk. Meanwhile, when features are correlated (dashed lines), the misspecified models do not yield zero excess risk neither according to 0-1 loss nor square loss.} 
\end{figure}
In this section, we construct simulations that illustrate classification-calibration in practice. We also present evidence on future investigations to characterize when surrogate loss minimization can recover optimal classifiers. 

In these simulations we construct features $X$ that are either normally or uniformly distributed, and that may or may not be correlated. According to a fixed true $\beta^*$, these ultimately determine true underlying probabilities $p^*$ of a binary outcome $Y$. We then construct $Y$ according to binomial draws of $p^*$. This data-generating process defines primitives $(\beta^*, p^*)$ that are unobserved by a machine learner, as well as data $(X,Y)$ that are observed. 

The machine learner constructs models of the data-generating process to predict $Y$ from $X$. We suppose the learner passes training instances of $(X,Y)$ through a procedure that minimizes a convex loss function $\phi$, either square or logistic loss, and thus produces estimated regressors $f_\phi(\langle \tilde \beta, X \rangle)$ in a test set. To measure the success of their procedure, we compute excess $\phi$-risk, $\bb{E}[\phi (Yf_\phi (X)] - \bb{E}[\phi (Yf^*(X)] $ as well as excess 0-1 risk,  $\bb{P}[ \sign (\langle \tilde \beta, X \rangle) \neq Y] - \bb{P}[ \sign (\langle  \beta^*, X \rangle) \neq Y] $.

We consider cases where $p^*$ is linear or nonlinear in $X$. This allows us to explore two kinds of misspecification: one where the models are only misspecified through a norm restriction, and the other where the estimating model itself is structurally different from the data-generating process.

\subsection{$p^*$ is linear in features}

We first consider the linear case $p^* = \frac{1}{2} + \beta^* X$ for which minimizing square loss produces a well-specified model. Recall that in Proposition \ref{prop:sq-loss-isotropic}, we showed that when $X$ is isotropic, then models minimizing square loss are classification calibrated. That is, they recover the optimal classifier even when their norm constraint $r$ prevents them from recovering the optimal regressor. Meanwhile, when features are not isotropic (Theorem \ref{thm:sq-loss-general}), then more restrictive choices of $r$ prevent the convergence of the classifications to the optimal. 

We demonstrate these results in Figure \ref{fig:sqloss_plot} where we plot excess 0-1 classification risk and excess $\phi$ risk. Features are distributed as $U(-\frac{1}{8}, \frac{1}{8})$ and they fix $p^*  = \frac{1}{2} + \langle  \beta^*, X \rangle $ with $\beta^* = (1,-3)$. Consider first when there is no binding norm restriction, $r = \infty$. The model is correctly specified regardless of whether features are correlated, and as the gray lines show, both 0-1 and $\phi$ excess risks converge to 0 as the training size grows. Meanwhile, when we impose a misspecified norm constraint on the models (red and orange lines), then $\phi$ excess risks no longer converge to 0. Yet 0-1 excess risk still does converge to 0 so long as the features are uncorrelated (solid red and orange lines), marking the cases where models are classification calibrated.   

We separately minimized logistic loss to model the same data generating process. While the functional form is misspecified in this case, we were surprised to see that a notion of Proposition \ref{prop:sq-loss-isotropic} still held. As seen in Figure \ref{fig:logloss_plot}, these models were again classification calibrated in the isotropic case. This suggests future exploration of a wider class of surrogate loss functions that yield optimal classifiers associated with linear predictors.

\begin{figure}[b]
  \centering
\includegraphics[width=.55 \columnwidth]{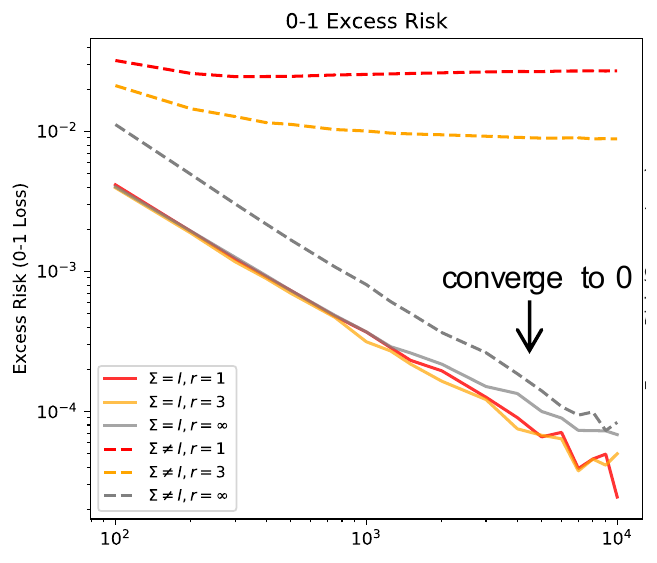}
\caption{\label{fig:logloss_plot}   Minimizing logistic loss when  $p^*$ is linear in $X$. Misspecified models in this case with a tight norm restriction (solid red and solid orange lines) are seen to be classification calibrated in the isotropic case, with excess 0-1 risk converging to zero. Meanwhile, tightening the norm restrictions leads to convergence toward non-zero values. This suggests that our results on square loss may be extended to other convex surrogate loss functions.}
\end{figure}
% \textcolor{red}{[Consider arguing convergence rate according to slope.]} 

\subsection{$p^*$ is nonlinear in features}

\begin{figure}[t]
  \centering
\includegraphics[width=\columnwidth]{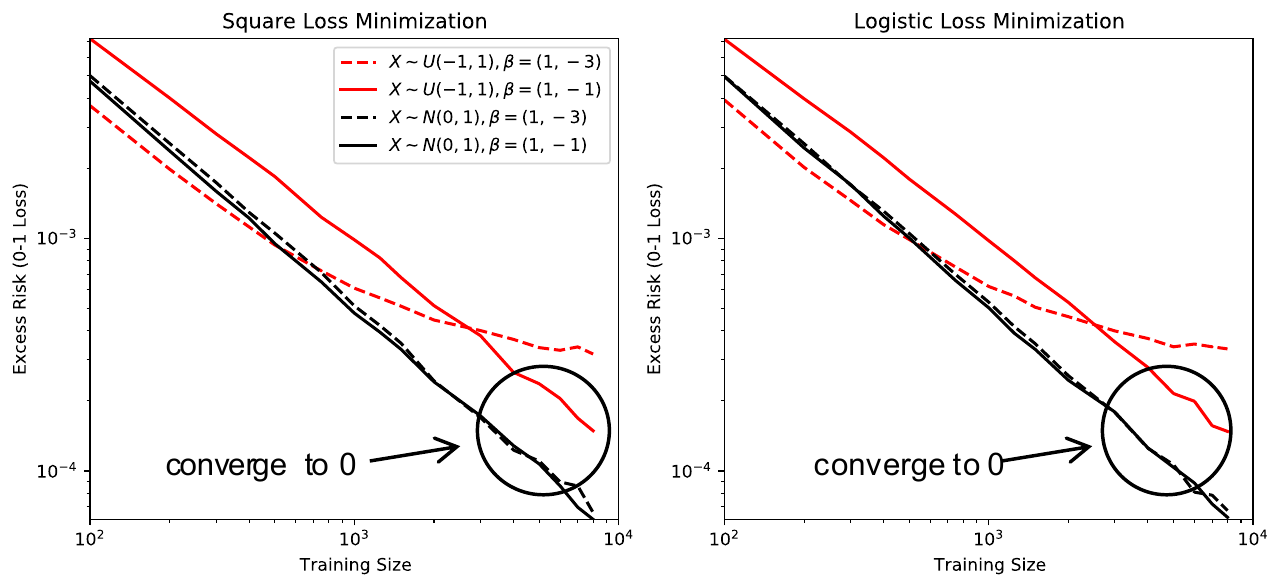}
\caption{\label{fig:rotinv_plot} Minimizing square and logistic loss when $p^*$ is a nonlinear transformation of $\langle \beta, X \rangle$. Adjusting the data-generating process is seen to affect convergence of excess 0-1 risk to zero, showing there is a limit to how much we can extend the result in Proposition \ref{prop:general-loss}. When features are uniformly distributed but $\beta$ is non-symmetric (dashed red line), models are not classification-calibrated regardless of whether $\phi$ is square or logistic loss. }
\end{figure}

We next explored weakening the assumption that $p^*$ is linear in $X$ to see whether we could extend the result in Proposition \ref{prop:general-loss}  to cases where the law of $X$ is not rotation invariant. We learned that this result cannot be generalized without further assumptions. 

In this new set of simulations, we adjusted the data-generating process by passing $\langle  \beta^*, X \rangle$ through the logistic CDF function to construct the true underlying probabilities $p^*$. The results are depicted in Figure \ref{fig:rotinv_plot}. We first considered specifications of $X$ satisfying rotation invariance: we constructed two features each independent and distributed as $N(0,1)$. For both symmetric and non-symmetric choices of $\beta^*$, minimizing square or logistic loss produced classification-calibrated models regardless of whether square or logistic loss was minimized (black lines in each panel), supporting the result in Proposition \ref{prop:general-loss}. However, when we instead constructed features that are independent but distributed as $U(-1,1)$, so that the law of $X$ is not rotation invariant, we saw that excess 0-1 risk does not necessarily converge to 0. This is depicted by the dashed red lines corresponding to $\beta^* = (1,-3)$. Therefore, to guarantee convergence to the optimal classifier when $p^*$ is not linear in $X$ and the law of $X$ is not rotation invariant, we understood that additional assumptions are required.  

% \textcolor{red}{[Add why we got convergence to 0 when $\beta$ symmetric.]} 

\section{Conclusion}

In this paper, we used a geometric distinction between classification and regression problems to more precisely characterize how loss in one setting relates to loss in the other. Using the scale invariance of classification, we were able to improve the bounds used by theorists and practitioners to compare classification procedures against one another. We hope that this work will help inform decisions about which classification algorithms to deploy in practice, and that it may open the door for effective new algorithms that aim to directly predict the direction of the conditional expectation function rather than its location.

\newpage

\bibliographystyle{plainnat}
\bibliography{outline}

\begin{thebibliography}{7}
\providecommand{\natexlab}[1]{#1}
\providecommand{\url}[1]{\texttt{#1}}
\expandafter\ifx\csname urlstyle\endcsname\relax
  \providecommand{\doi}[1]{doi: #1}\else
  \providecommand{\doi}{doi: \begingroup \urlstyle{rm}\Url}\fi

\bibitem[Bartlett et~al.(2006)Bartlett, Jordan, and
  McAuliffe]{doi:10.1198/016214505000000907}
Peter~L Bartlett, Michael~I Jordan, and Jon~D McAuliffe.
\newblock Convexity, classification, and risk bounds.
\newblock \emph{Journal of the American Statistical Association}, 101\penalty0
  (473):\penalty0 138--156, 2006.
\newblock \doi{10.1198/016214505000000907}.

\bibitem[Hazan et~al.(2014)Hazan, Koren, and Levy]{pmlr-v35-hazan14a}
Elad Hazan, Tomer Koren, and Kfir~Y. Levy.
\newblock Logistic regression: {{Tight}} bounds for stochastic and online
  optimization.
\newblock In Maria~Florina Balcan, Vitaly Feldman, and Csaba Szepesv{\'a}ri,
  editors, \emph{Proceedings of the 27th Conference on Learning Theory},
  volume~35 of \emph{Proceedings of Machine Learning Research}, pages 197--209,
  {Barcelona, Spain}, June 2014. {PMLR}.

\bibitem[Koltchinskii and Beznosova(2005)]{10.1007/11503415_20}
Vladimir Koltchinskii and Olexandra Beznosova.
\newblock Exponential convergence rates in classification.
\newblock In Peter Auer and Ron Meir, editors, \emph{Learning Theory}, pages
  295--307, {Berlin, Heidelberg}, 2005. {Springer Berlin Heidelberg}.
\newblock ISBN 978-3-540-31892-7.

\bibitem[Lugosi and Vayatis(2004)]{lugosi_bayes-risk_2004}
G{\'a}bor Lugosi and Nicolas Vayatis.
\newblock On the {{Bayes}}-risk consistency of regularized boosting methods.
\newblock \emph{The Annals of Statistics}, 32\penalty0 (1):\penalty0 30--55,
  February 2004.
\newblock ISSN 0090-5364, 2168-8966.
\newblock \doi{10.1214/aos/1079120129}.

\bibitem[Mammen and Tsybakov(1999)]{mammen_smooth_1999}
Enno Mammen and Alexandre~B. Tsybakov.
\newblock Smooth discrimination analysis.
\newblock \emph{The Annals of Statistics}, 27\penalty0 (6):\penalty0
  1808--1829, December 1999.
\newblock ISSN 0090-5364, 2168-8966.
\newblock \doi{10.1214/aos/1017939240}.

\bibitem[Rosasco et~al.(2004)Rosasco, De~Vito, Caponnetto, Piana, and
  Verri]{10.1162/089976604773135104}
Lorenzo Rosasco, Ernesto De~Vito, Andrea Caponnetto, Michele Piana, and
  Alessandro Verri.
\newblock Are loss functions all the same?
\newblock \emph{Neural Computation}, 16\penalty0 (5):\penalty0 1063--1076, May
  2004.
\newblock ISSN 0899-7667.
\newblock \doi{10.1162/089976604773135104}.

\bibitem[Zhang(2004)]{zhang_statistical_2004}
Tong Zhang.
\newblock Statistical behavior and consistency of classification methods based
  on convex risk minimization.
\newblock \emph{The Annals of Statistics}, 32\penalty0 (1):\penalty0 56--85,
  February 2004.
\newblock ISSN 0090-5364, 2168-8966.
\newblock \doi{10.1214/aos/1079120130}.

\end{thebibliography}

\newpage

\section{Appendix}

\subsection{Proof of Lemma \ref{lem:quad-form-sigma}}

%\textit{Proof of Lemma \ref{lem:quad-form-sigma}}
\begin{proof}
  We can expand
  \begin{align*}
    \bk{X}{v}^2 
    &= \left(\sum_{i=1}^d X_iv_i \right)^2 \\
    &= \sum_{i=1}^d \sum_{j=1}^d X_iX_jv_iv_j \\
    &= v^\top X X^\top v.
  \end{align*}
  Now take expectations and use $\bb{E}[XX^\top] = \Sigma$.
\end{proof}

\subsection{Proof of Lemma \ref{lem:pop-loss}}

%\textit{Proof of Lemma \ref{lem:pop-loss}}

\begin{proof}
 % Let $\bks{\beta^*}{X}$ be the orthogonal projection of $Y$ onto the linear subspace $\Set[\bks{\gamma}{X}]{\gamma \in \bb{R}^d}$ in $L^2(\bb P)$. 
By definition of $\beta^*$ as the orthogonal projection of $Y$ onto the span of the features, the first equality follows from the Pythagorean theorem in $L^2(\bb P)$. Then, the second follows from appying Lemma \ref{lem:quad-form-sigma} to $\bb{E}(\bks{\beta^*}{X}-\bks{\beta}{X})^2  = \bb{E}(\bks{\beta^*-\beta}{X})^2 $.
\end{proof}

\subsection{Proof of Corollary \ref{cor-beta-min-distance}}

%\textit{Proof of Corollary \ref{cor-beta-min-distance}}
\begin{proof} 
Following Lemma \ref{lem:pop-loss}, the mean square loss is minimized when $ (\beta^* - \beta)^\top I (\beta^* - \beta) = \norm{\beta^* - \beta}_2^2$ is minimized.
\end{proof}

\subsection{Proof of Theorem \ref{prop:sq-loss-isotropic}}

%\textit{Proof of Theorem \ref{prop:sq-loss-isotropic}}

\begin{proof}
  By Lemma \ref{lem:pop-loss}, we know that $\tilde \beta$ minimizes
  \[\beta \mapsto ( \beta - \beta^*)^\top\Sigma(\beta - \beta^*)\] subject to $\norms{ \beta} \le r$. This is a convex program; from the KKT conditions, we obtain that %\textcolor{red}{[Add derivation]} 
  \[\tilde \beta = (\Sigma + \lambda)^{-1}\Sigma \beta^*\] where $\lambda = \lambda(r,\Sigma,\beta^*) \ge 0$ is a Lagrange multiplier.
  \begin{proof}[Derivation]
%    \textcolor{green}{I think this should go in appendix}
    This is a feasible convex program as seen by taking $\beta = 0$, so the KKT conditions are necessary.
    The Lagrangian is
    \[\c L(\beta,\lambda) =  ( \beta - \beta^*)^\top\Sigma(\beta - \beta^*) + \lambda(\beta^\top\beta-r^2)\]
    The dual feasibility condition is $\lambda \ge 0$ and the stationarity condition is
    \begin{equation*}
      \frac{\partial \c L}{\partial \beta} = 0
      \iff 2\Sigma^\top(\beta - \beta^*) + 2\lambda\beta = 0
      \iff \beta = (\Sigma + \lambda)^{-1}\Sigma\beta^*
    \end{equation*}
    where we used that $\Sigma^\top = \Sigma$. This proves that the stated conditions must hold at the constrained minimizer $\tilde\beta$.
  \end{proof}
  \noindent We therefore wish to bound
  \begin{align*}
   & \inf_{t \in \bb{R}} \norm{t(\Sigma + \lambda)^{-1}\Sigma\beta^* - \frac{\beta^*}{\norms{\beta^*}_2}}_2.
   \intertext{Making the change of variable $t = t'/\norms{\beta^*}$ this is}
   = & \inf_{t' \in \bb{R}} \norm{t'(\Sigma + \lambda)^{-1}\Sigma\frac{\beta^*}{\norms{\beta^*}_2} - \frac{\beta^*}{\norms{\beta^*}_2}}_2\\
   \le & \inf_{t \in \bb{R}} \norm{t(\Sigma + \lambda)^{-1}\Sigma- I}_{\mathrm{op}}\\% \textcolor{red}{\text{Needs explanation}}\\ 
   \le & \inf_{t \in \bb{R}} \sup_{\lambda \ge 0} \norm{t(\Sigma + \lambda)^{-1}\Sigma- I}_{\mathrm{op}}
   \intertext{Since the operator norm is equal to the maximum absolute eigenvalue, this is precisely}
  = & \inf_{t \in \bb{R}} \sup_{\lambda \ge 0} \max_{1 \le k \le d} \left|\frac{t\sigma_k}{\sigma_k + \lambda} - 1\right|.%\textcolor{red}{\text{Needs explanation, what's sigma}}
  \intertext{Choosing $t = (1+\lambda)$ and interchanging $\max$ and $\sup$, this is}
  \le &  \max_{1 \le k \le d} \sup_{\lambda \ge 0} \left|\frac{(1+\lambda)\sigma_k}{\sigma_k + \lambda} - 1\right|.
  \intertext{It can be seen by differentiating that for each $\sigma_k$ the inner expresson is monotone increasing in $\lambda$ from $0$ to $|\sigma_k - 1|$. Therefore, the above expression evaluates to }
  = & \max_{1 \le k \le d} \left|\sigma_k - 1\right| \\
  = & \norm{\Sigma - I}_{\mathrm{op}}.
  \end{align*}
  Finally, this entire argument holds when $\Sigma$ is replaced by $a \Sigma$ for any $a > 0$. Optimizing over all $a$, taking the limit as $a \downarrow 0$ if need be, gives the announced result. %\textcolor{red}{[Elaborate on this argument]}
\end{proof}

\subsection{Proof of Proposition \ref{prop:general-loss}}

%\textit{Proof of Proposition \ref{prop:general-loss}}
\begin{proof}
  Firstly, note that since $Y \in \set{\pm 1}$ we may write
  \begin{align*}
    \bb{E}\phi(Y,\bks{\beta}{X}) &=
    \bb{E}\left[ \mathbbm{1}\{Y=1\} \phi(1,\bks{\beta}{X}) +(1-\mathbbm{1}\{Y=1\})\phi(-1,\bks{\beta}{X})\right].
    \intertext{Iterating expectations then yields}
    &= \bb{E}\left[ \bb{P}(Y=1|X) \phi(1,\bks{\beta}{X}) +(1-\bb{P}(Y=1|X))\phi(-1,\bks{\beta}{X})\right] \\
    &=  \bb{E}\left[  \eta(\bks{\beta^*}{X})\phi(1,\bks{\beta}{X}) +(1-\eta(\bks{\beta^*}{X}))\phi(-1,\bks{\beta}{X})\right] \\
    &\equiv \bb{E}F_{\eta,\phi}(\bks{\beta^*}{X},\bks{\beta}{X}),
    \intertext{by \ref{it:suff}. Now, let $g \in O_d(\bb R)$ be any rotation such that $g\beta^* = \beta^*$. Note that these rotations form a group, $G$, for which $\Sets[c\beta^*]{c \in \bb{R}}$ is the unique invariant subspace. Moreover, for any such rotation, $g^\top = \inv g$ is also in $G$. By \ref{it:rotinv} we may then write}
    &= \bb{E}F_{\eta,\phi}(\bks{\beta^*}{g^\top X},\bks{\beta}{g^\top X}) \\
    &= \bb{E}F_{\eta,\phi}(\bks{g \beta^*}{X},\bks{g\beta}{X}) \\
    &= \bb{E}F_{\eta,\phi}(\bks{\beta^*}{X},\bks{g\beta}{X}).
  \end{align*}
  Note also that $\norms{g \tilde \beta} = \norms{\tilde \beta}$, so the set of constrained minimizers must contain $\Sets[g\tilde \beta]{g \in G}$.
  However, the minimizer of a convex function over a strictly convex set is unique, so by \ref{it:conv} we must have $g\tilde\beta = \tilde \beta$ for all $g \in G$. 
  Thus $\tilde \beta$ must belong to the unique invariant subspace of $G$, from which we conclude that $\tilde \beta = c\beta^*$.

\end{proof}

\subsection{Elaborating on remark about excess classification risk  in rotationally invariant case}

%\begin{remark} Note that when we consider linear predictors and rotationally invariant features, the $L^2(\bb P)$ norm is merely a rescaling of the standard Euclidean norm, so the two notions of angle coincide. In this case, a computation that uses independence of the magnitude and direction of $X$ shows that the excess classification error is precisely 
 % \[\frac{1}{\pi}\left(\int_{0}^\theta \sin t\, dt \right)\bb{E}|\bk{X}{\beta^*}|.\] This suggests that the dependence on $\norm{f^*}_{2, \bb P}$ could perhaps be improved to $\bb{E}|f^*(X)|$, although the linear dependence on the scale of $f^*$ is necessary. More interestingly, we have 
  %\[\int_{0}^\theta \sin t\, dt = 1 - \cos \theta \approx \theta^2\] for $\theta \ll 1$. On the contrary, we have $\sin\theta \approx \theta$ appearing in our bound, which in turn dominates that of \citet{doi:10.1198/016214505000000907}. Thus, both are off by at least a power of 2 for rotationally invariant distributions! This implies that fast rates occur generically for rotation invariant linear classification even in absence of any margin condition \emph{and} suggests that the standard bounds are unable to capture this phenomenon, since information-theoretic lower bounds generally will imply that $\sin \theta \gtrsim 1/\sqrt{n}$. 
%\end{remark}

\begin{proof}
  Under rotational invariance, we have that $\norm{X}$ is independent of $Z = X/\norm{X}$. We can compute using \eqref{eq:cls-risk-step} that
  \begin{align*}
 & \bb{E}[|\bk{\beta^*}{X}|\mathbbm{1}\{\sign \bk{\beta^*}{X} \ne \sign \bk{\beta}{X}] \\
  &= \bb{E}[\norm{\beta^*}_2\norm{X}_2|\cos \theta(Z,\beta^*)|\mathbbm{1}\{Z \in S\}] \\
  &= \norm{\beta^*}_2\bb{E}\norm{X}_2 \bb{E}[|\cos \theta(Z,\beta^*)|\mathbbm{1}\{Z \in S\}] \\
  &= \norm{\beta^*}_2\bb{E}\norm{X}_2 \left(\frac{1}{\pi}\int_{0}^{\theta(\beta,\beta^*)} \sin t\, dt\right),
  \end{align*}
  where the last step uses the fact that $Z$ is uniformly distributed on the sphere. Finally, a similar computation shows that $\norm{\beta^*}_2\bb{E}\norm{X}_2 = \frac{\pi}{2}\bb{E}|\bk{\beta^*}{X}|$, completing the proof of our claim. 
\end{proof}

\end{document}